%% file: main.tex
\documentclass{article} 
\PassOptionsToPackage{sort&compress}{natbib}
\usepackage{iclr2026_conference,times}
\setcitestyle{numbers,square,citesep={,}}

\input{math_commands.tex}

\usepackage[hidelinks]{hyperref}
\usepackage{url}
\usepackage{amsmath}
\usepackage{amssymb}
\usepackage{amsthm}
\usepackage{booktabs}
\usepackage{array}
\usepackage{multirow}
\usepackage{placeins} 
\usepackage{graphicx}
\graphicspath{{figures/}}
\usepackage{enumitem}
\usepackage{tikz}

\theoremstyle{plain}
\newtheorem{theorem}{Theorem}
\newtheorem{proposition}{Proposition}
\newtheorem{lemma}{Lemma}
\newtheorem{corollary}{Corollary}
\theoremstyle{definition}
\newtheorem{definition}{Definition}
\newtheorem{hypothesis}{Hypothesis}
\theoremstyle{remark}
\newtheorem{remark}{Remark}

\providecommand{\dd}{\mathrm{d}}
\providecommand{\eexp}{\mathrm{e}}
\providecommand{\smin}{\sigma_{\min}}
\providecommand{\smax}{\sigma_{\max}}
\providecommand{\Dc}{D_{\mathrm{c}}}
\providecommand{\Du}{D_{\mathrm{u}}}
\providecommand{\Dw}{D_{w}}
\providecommand{\etac}{\eta_{\mathrm{c}}}
\providecommand{\etau}{\eta_{\mathrm{u}}}
\providecommand{\etaw}{\eta_{w}}
\providecommand{\verify}[1]{} 

\def\eqref#1{(\ref{#1})}
\def\Eqref#1{Equation~(\ref{#1})}

\newif\ifarxiv
\arxivfalse

\makeatletter
\providecommand{\input@path}{}
\g@addto@macro\input@path{{../}}
\makeatother

\title{Guidance Breaks the Fitted Operator:\\
A Terminal-Fitted Repair for Classifier-Free Guidance}

\author{Shiheng Zhang \\
University of Washington \\
\texttt{shzhang3@uw.edu}}

\iclrfinalcopy

\begin{document}

\pagestyle{plain}

\maketitle

\input{sections/00_abstract.tex}
\input{sections/01_introduction.tex}

\input{guided_sections_2_3.tex}
\input{guided_empirical_certificates.tex}

\input{sections/045_related_work.tex}

\input{sections/05_discussion.tex}

\input{sections/06_reproducibility.tex}

\bibliographystyle{iclr2026_conference}
\bibliography{references}

\input{sections/99_appendix.tex}

\end{document}

%% file: math_commands.tex
\usepackage{amsmath,amsfonts,bm}

\def\eqref#1{equation~\ref{#1}}
\def\Eqref#1{Equation~\ref{#1}}

\def\1{\bm{1}}

\DeclareMathAlphabet{\mathsfit}{\encodingdefault}{\sfdefault}{m}{sl}
\SetMathAlphabet{\mathsfit}{bold}{\encodingdefault}{\sfdefault}{bx}{n}

\newcommand{\E}{\mathbb{E}}

\newcommand{\R}{\mathbb{R}}



%% file: sections/00_abstract.tex
\begin{abstract}
Classifier-free guidance (CFG) is the standard way to strengthen
class-conditioning in diffusion and flow-matching samplers, yet at large
guidance it oversaturates and destabilizes---symptoms practitioners suppress
with more steps or limited-interval schedules. We analyze CFG through an
asymptotic-preserving, numerical-analysis lens. Building on a
recent result \citep{ap2026} that the deterministic DDIM step is the unique
\emph{fitted} operator for the unguided terminal layer---exact on the final,
small-$\sigma$ stretch of sampling---we show
that guidance re-stiffens \emph{exactly} the discriminative subspace to an
anomalous exponent $1+w$ (guided coordinates contract like $\sigma^{1+w}$
rather than $\sigma$). DDIM is therefore no longer fitted there, and on coarse
meshes its guided residual diverges as $\smin\to0$. We prove a guided clock barrier with three ordered step-size
thresholds, and read one-step oversaturation as its endpoint---a solver artifact
on the calibration model rather than the continuous guided law. The same analysis
yields a one-coefficient, zero-extra-NFE repair: replace CFG's $w(r-1)$ by
$r^{1+w}-r$ on the guidance direction. This coefficient is the unique
spectrum-free terminal-exact one, and it preserves the sign of every analyzed
coordinate. On the calibration model's discriminative
crossover it removes CFG's $\smin$-divergent blow-up and is first-order accurate
against the exact guided flow as $\smin\to0$---asymptotic-preserving in $\smin$,
though not uniform in $w$. On learned CIFAR-10 checkpoints---and, as a cross-domain smoke test, on
Stable Diffusion~1.5 DDIM---it acts as a high-guidance stabilizer at no extra
cost rather than a universal quality knob: it cuts residual amplification and
saturation and gives $9/9$ point-FID wins over CFG on the tested grid, while in
the hard-cell blocks its classifier-proxy target accuracy stays close to CFG and
terminal guidance shutdown loses much more. We report the limits alongside: it
is not a universal image-quality win (KID can favor CFG; an interval can win FID
in some cells), and against a dense vanilla-CFG reference it is not a uniformly
better integrator of that field.
\end{abstract}

%% file: sections/01_introduction.tex
\section{Introduction}
\label{sec:intro}

Diffusion and flow-matching samplers integrate a probability-flow ODE whose
velocity field stiffens as the noise scale $\sigma\to\smin$: on a coordinate
normal to the data manifold the exact flow contracts at exponent one, by
$\sigma_{n+1}/\sigma_n$ per step. A recent fitted-operator analysis of
unguided samplers~\citep{ap2026} shows that this terminal
layer admits a \emph{unique} fitted (layer-exact) one-step operator: a
frozen-field Euler step---Euler with the velocity held at its start-of-step
value---reproduces that exact contraction if and only if its integration
variable is affine in $\sigma$. This singles out the $\sigma$-clock step,
algebraically identical to deterministic DDIM \citep{song2021ddim}, and makes
it \emph{asymptotic-preserving} (AP) on the exactly solvable calibration
models studied there~\citep{ap2026}---uniformly accurate as $\smin\to0$. This paper asks
one question of that lens: \emph{what survives of DDIM's fitted-operator
protection once we turn on guidance?}

\paragraph{Guidance re-stiffens the discriminative subspace.}
Classifier-free guidance (CFG) \citep{hosalimans2022cfg} runs the sampler on an
extrapolated denoiser $\Dw=(1+w)\Dc-w\Du$, that is, on a second velocity field
$\etaw=\etac+w(\etac-\etau)$ (the $\eta$'s are noise-prediction fields;
Section~\ref{sec:guided-flow}) with guidance scale $g=1+w$. This is not a uniform
stiffening. On the commuting Gaussian model of
Section~\ref{sec:guided-flow}, the guided terminal exponent
$\mu_w(\sigma)$ splits by layer (a \emph{layer} here is a regime of
state-space directions, not a network layer; Figure~\ref{fig:layers}): it
stays exactly one
on directions shared by the class and the marginal, freezes to zero on
tangential directions, and rises to $1+w$ on the \emph{discriminative}
directions---the subspace that separates the class from the marginal
(Proposition~\ref{prop:glayers}). Guidance re-stiffens precisely this
discriminative subspace, and it does so with a \emph{second} singular
parameter: $w$ enters the terminal structure on the same footing as $\smin$.
Because $\Dw$ is not any distribution's posterior mean, DDIM's fitted-operator
protection lapses there: on coarse meshes, vanilla DDIM$+$CFG leaves the
class of fitted operators and its terminal residual diverges as $\smin\to0$.

\begin{figure}[t]
\centering
\begin{tikzpicture}[scale=0.72,>=stealth]
  \draw[->] (0,0) -- (8.8,0);
  \draw[->] (0,0) -- (0,3.5) node[left]{\small $\mu_w(\sigma)$};
  \node[below] at (7.1,-0.10) {\small $\lambda=-\log\sigma$\;\;(terminal $\rightarrow$)};
  \draw[dashed,gray] (0,3) -- (8.5,3);
  \node[right] at (8.5,3) {\small $1+w$};
  \node[right] at (8.5,1) {\small $1$};
  \draw[dotted] (4.1,0) -- (4.1,2.0);
  \node[fill=white,inner sep=1.5pt] at (4.1,0.72) {\small $\sigma^2\!=\!b$};
  \draw[thick,black] (0,1) -- (8.5,1);
  \node at (6.7,1.28) {\small shared normal ($a{=}b{=}0$)};
  \draw[very thick,black] plot[smooth] coordinates
    {(0,1.05) (1.5,1.12) (3,1.45) (4.1,2.0) (5.2,2.6) (6.5,2.9) (8.5,2.98)};
  \node at (6.7,2.52) {\small discriminative ($a{=}0{<}b$)};
  \draw[thick,black,densely dashed] plot[smooth] coordinates
    {(0,0.95) (1.5,0.85) (3,0.60) (4.1,0.42) (5.2,0.25) (6.5,0.10) (8.5,0.03)};
  \node at (2.3,0.22) {\small tangential ($0{<}a{\le}b$)};
\end{tikzpicture}
\caption{The guided terminal exponent $\mu_w(\sigma)$ by layer type on the
commuting model (schematic; noise decreases to the right). Guidance leaves the
shared-normal exponent at one and freezes tangential directions, but
re-stiffens the discriminative directions to $1+w$---the subspace on which
DDIM$+$CFG exits the class of fitted operators.  Here $a,b$ are the class and
marginal variances along the direction (Section~\ref{sec:guided-flow}).}
\label{fig:layers}
\end{figure}
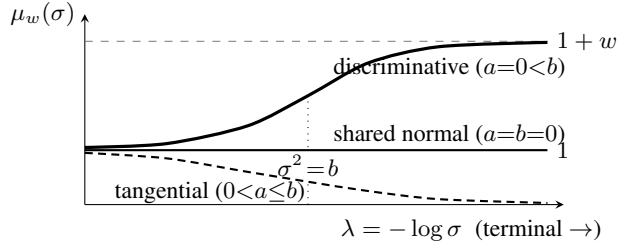

\paragraph{Two orthogonal bills.}
This is a statement about the \emph{solver}. A parallel line studies how the
continuous guided law itself departs from the tilted conditional it is meant
to sample \citep{bradley2024cfg,chidambaram2024cfg}: that charges the
\emph{model}; we charge the \emph{discretization}---our reference solution is
always the exact guided flow's own pushforward. The two bills are orthogonal,
and conflating them has obscured how much of guidance's terminal
pathology---the oversaturation and norm inflation seen at large $w$
\citep{sadat2025apg}---has a solver-induced component that appears even when
the exact guided flow is the reference, from stepping guidance with an
operator it has quietly un-fitted.

\paragraph{The repair is one coefficient.}
Rectifying the anomalous exponent $1+w$ by a gauge change and pushing the
$\sigma$-clock step through it yields a one-coefficient modification of CFG that
touches nothing else in the sampler---only the coefficient on the guidance
direction $\Du-\Dc$ changes:
\begin{align*}
  \text{CFG:}    \quad & x_{n+1}=\Dc+r\,(x_n-\Dc)+w(r-1)\,(\Du-\Dc),\\
  \text{fitted:} \quad & x_{n+1}=\Dc+r\,(x_n-\Dc)+\bigl(r^{1+w}-r\bigr)(\Du-\Dc),
\end{align*}
with $r=\sigma_{n+1}/\sigma_n$. It uses the same two denoiser evaluations per step---the
implementation delta is a single line, at zero additional NFE (network
function evaluations)---agrees with CFG to first order, and reduces to DDIM
at $w=0$. We call it the guided fitted step,
or \emph{fitted CFG}. It is the \emph{unique} scalar coefficient that is exact on the
discriminative terminal layer using no spectral information
(Lemma~\ref{lem:g2-unique}); it never reflects a coordinate across the class
manifold (Proposition~\ref{prop:gfit-sign}); and its one-step $r\to0$ limit
lands on the conditional denoiser $\Dc$ (the terminal-limit class projection),
where vanilla CFG lands on the overshot $\Dw$, the discrete mechanism of
oversaturation.

\paragraph{Contributions.}
\begin{itemize}[leftmargin=1.4em,itemsep=2pt,topsep=2pt]
\item \textbf{Structure and a barrier.} A layer classification of the guided
  terminal exponent (Proposition~\ref{prop:glayers}) and a guided clock barrier
  (Theorem~\ref{thm:gbarrier}) with three ordered step-size
  thresholds---reflection, residual-amplification failure
  (Section~\ref{sec:guided-program}), absolute stability---and the resulting
  two-tier guidance tax on step count (Corollary~\ref{cor:tax}).
\item \textbf{Oversaturation as an endpoint, and the guidance interval.} A
  model-free identity reading one-step DDIM$+$CFG's overshoot to $\Dw$ as the
  barrier's $r\to0$ endpoint (Remark~\ref{rem:onejump}), and a parameter-free
  prediction of the \emph{terminal} edge of limited-interval guidance
  \citep{kynkaanniemi2024interval} (Remark~\ref{rem:interval}).
\item \textbf{A spectrum-free fitted repair.} The guided fitted step
  (Definition~\ref{defn:gfitted}): first-order consistent, sign-preserving and
  terminal-exact on the model, unique (Lemma~\ref{lem:g2-unique}), and---on the
  discriminative crossover---free of the $\smin$-divergent blow-up and
  first-order accurate as $\smin\to0$ (Proposition~\ref{prop:g2-crossover-tax},
  Theorem~\ref{thm:cross-ap}): asymptotic-preserving in $\smin$, though not
  uniform in $w$.
\item \textbf{A scoped empirical program.} On learned CIFAR-10 \textsc{edm}
  \citep{karras2022edm} checkpoints, fitted CFG is a zero-extra-NFE high-guidance
  stabilizer: the residual and clipping certificates improve in both
  $w\in\{6.5,8\}$ cells; the
  hard-cell image evaluations (two 5k blocks, a 50k replication) give it the best
  FID with target accuracy preserved, and a DINOv2 backbone-swap keeps the same
  fitted-favoring ordering; a nine-cell grid gives $9/9$ FID wins over CFG. We
  report the
  limits alongside---the CFG-favoring KID split and threshold-derived interval
  baselines that win FID/KID at lower conditionality
  (Section~\ref{sec:checkpoint-certs}).
\end{itemize}

Section~\ref{sec:discussion} collects the claim boundaries.

%% file: guided_sections_2_3.tex
%
%
%

\providecommand{\E}{\mathbb{E}}
\providecommand{\R}{\mathbb{R}}
\providecommand{\eexp}{\mathrm{e}}
\providecommand{\dd}{\mathrm{d}}
\providecommand{\smin}{\sigma_{\min}}
\providecommand{\smax}{\sigma_{\max}}
\providecommand{\Dc}{D_{\mathrm{c}}}
\providecommand{\Du}{D_{\mathrm{u}}}
\providecommand{\Dw}{D_{w}}
\providecommand{\etac}{\eta_{\mathrm{c}}}
\providecommand{\etau}{\eta_{\mathrm{u}}}
\providecommand{\etaw}{\eta_{w}}
\providecommand{\Eqref}[1]{Equation~(\ref{#1})}

\ifcsname ifarxiv\endcsname\else
  \expandafter\newif\csname ifarxiv\endcsname
  \csname arxivtrue\endcsname
\fi

\section{The guided flow and its terminal layers}\label{sec:guided-flow}

\paragraph{The unguided flow.}
We work throughout in the variance-exploding family: $p_\sigma = p_{\rm
data} * \mathcal N(0,\sigma^2 I)$ for $\sigma \in [\smin,\smax]$, with
denoiser $D(x,\sigma) = \E[x_0 \mid x_\sigma = x]$ and normalized residual
$\eta = (x - D)/\sigma$---the noise prediction $\hat\varepsilon$ of practice
($-\sigma\nabla_x \log p_\sigma$); not DDIM's stochasticity parameter, which
is $0$ throughout. The probability flow ODE is $\dd X/\dd\sigma =
\eta(X,\sigma)$, and the decomposition $x = D + \sigma\eta$ splits the state
into a slow manifold component and a stiff normal component---normal to the
data manifold, contracting fastest---whose exponent
is $1$: on a normal coordinate the exact flow contracts by
$(\sigma_{n+1}/\sigma_n)^1$ per step. A fitted-operator analysis of the
unguided terminal layer~\citep{ap2026} shows
that a frozen-field Euler step is layer-exact if and
only if its clock is affine in $\sigma$, making the $\sigma$-clock
step---algebraically the deterministic DDIM update \citep{song2021ddim}---the
unique fitted operator
up to affine reparameterization, with rectified flow its flow-matching
counterpart; this fitted-operator property is what makes it
asymptotic-preserving as $\smin \to 0$. The present paper asks what
survives of that protection under guidance.

\paragraph{Guidance as a second velocity field.}
Classifier-free guidance \citep{hosalimans2022cfg} replaces the conditional
denoiser $\Dc(x,\sigma)$ by the extrapolation
\begin{equation}\label{eq:guided-denoiser}
  \Dw \;=\; (1+w)\,\Dc \;-\; w\,\Du,
  \qquad w \ge 0,
\end{equation}
where $\Du$ is the unconditional (marginal) denoiser and $g = 1+w$ is the
guidance scale of practice, so $w=6.5$ is CFG scale $g=7.5$. Writing $\etac = (x-\Dc)/\sigma$ and $\etau =
(x-\Du)/\sigma$, the sampled dynamics is the \emph{guided flow}
\begin{equation}\label{eq:guided-flow}
  \frac{\dd X}{\dd \sigma} \;=\; \etaw
  \;=\; (1+w)\,\etac - w\,\etau
  \;=\; \etac + w\,\gamma,
  \qquad
  \gamma \;:=\; \etac - \etau \;=\; \frac{\Du - \Dc}{\sigma},
\end{equation}
with $\gamma$ the normalized guidance direction. \Eqref{eq:guided-flow}
is the dynamical statement of guidance: the sampler
follows the conditional velocity plus $w$ times the discriminative
correction. Two facts follow. First, $\Dw$ is in
general not a VE posterior mean with a positive-semidefinite Tweedie
covariance, so $\etaw$ is not the residual of such a denoiser; the rigidity
theorem that singles out DDIM as layer-exact inside the
unguided denoiser class~\citep{ap2026} no longer applies directly. The exit is
quantitative, not merely definitional: for any true denoiser,
$J_D = \nabla_x D =
\mathrm{Cov}(X_0 \mid x)/\sigma^2 \succeq 0$ \citep{efron2011tweedie} caps
the exponent of the residual field at $1$ (on the model, the residual exponent
is $A \le 1$ in \eqref{eq:AB}), while on discriminative directions the guided
exponent exceeds $1$ at every $\sigma$
(Proposition~\ref{prop:glayers})---no repackaging of any denoiser, affine
or otherwise, reproduces the guided field.
Second, the guided flow carries a
\emph{second} parameter: any AP statement must state its $w$-dependence,
and the fitted step below removes the $\smin$-divergent terminal barrier
for bounded $w$, though not uniformly in $w$.

\paragraph{The model problem.}
Our analysis is carried out on the guided analog of the unguided model
problem~\citep{ap2026}.

\begin{hypothesis}[commuting class--marginal pair]\label{hyp:commuting}
$p_{\mathrm c} = \mathcal N(0, C_{\mathrm c})$ and $p_{\mathrm u} =
\mathcal N(0, C_{\mathrm u})$ with $[C_{\mathrm c}, C_{\mathrm u}] = 0$,
and on the shared eigenbasis the spectra satisfy $a_i \le b_i$.
\end{hypothesis}

The ordering $a_i \le b_i$ is the \emph{class-subset hypothesis}: the class
is narrower than the marginal in every probed direction (as when the
marginal is a mixture whose components share the class covariance). Real
learned checkpoints need not obey it; the reverse case is a genuine model
boundary, discussed in Remark~\ref{rem:reverse}. On an eigen-coordinate
$\xi$ with class variance
$a$ and marginal variance $b$, the two denoisers are linear and the
residuals are
\begin{equation}\label{eq:AB}
  \etac = A\,\frac{\xi}{\sigma}, \qquad
  \etau = B\,\frac{\xi}{\sigma}, \qquad
  A := \frac{\sigma^2}{a+\sigma^2}, \quad
  B := \frac{\sigma^2}{b+\sigma^2}, \quad
  0 \le B \le A \le 1 .
\end{equation}
The guided flow \eqref{eq:guided-flow} therefore closes coordinate-wise,
\begin{equation}\label{eq:mu}
  \frac{\dd \log \xi}{\dd \log \sigma}
  \;=\; \mu_w(\sigma)
  \;:=\; (1+w)A - wB
  \;=\; A + w\,(A - B),
\end{equation}
and every question about clocks, stability, and fitted operators reduces to
the behavior of the \emph{exponent function} $\mu_w$.

\begin{proposition}[layer classification under guidance]\label{prop:glayers}
Under Hypothesis~\ref{hyp:commuting}, as $\sigma \to 0$ the exponent
function \eqref{eq:mu} satisfies:
(i) on shared normal directions ($a = b = 0$), $\mu_w \equiv 1$: the
guidance weight cancels and the layer is the unguided exponent-$1$
layer;
(ii) on tangential directions ($0 < a \le b$), $\mu_w \to 0$: the
coordinate freezes;
(iii) on \emph{discriminative normal} directions ($a = 0 < b$),
\begin{equation}\label{eq:crossover}
  \mu_w(\sigma) \;=\; 1 + w\,\frac{b}{b+\sigma^2}
  \;\longrightarrow\; 1+w ,
\end{equation}
monotonically as $\sigma$ decreases, with crossover midpoint at
$\sigma^2 = b$.
\end{proposition}

\newcommand{\ProofGLayers}{%
For $a=0$ one has $A \equiv 1$, so $\mu_w = (1+w) - wB = 1 +
w(1-B) = 1 + wb/(b+\sigma^2)$; the limits and monotonicity are immediate.
Cases (i) and (ii) follow from $A, B \to 1$ and $A, B \to 0$ respectively.
}
\ifarxiv\begin{proof}\ProofGLayers\end{proof}\else\noindent\emph{(Proofs
for Sections~\ref{sec:guided-flow} and~\ref{sec:guided-program} are
deferred to Appendix~\ref{app:proofs}.)}\fi

Proposition~\ref{prop:glayers} is the structural fact of the paper.
Guidance does not stiffen the flow uniformly; it re-stiffens \emph{exactly
the discriminative subspace}---the directions collapsed within the class but
present in the marginal ($a=0<b$)---and
\eqref{eq:crossover} locates where the re-stiffening begins: the anomalous
exponent switches on as $\sigma$ descends through the class-separation
scale $\sqrt b$. Above that scale the guided flow is, to leading order, the
unguided flow; below it, the terminal layer runs at exponent $1+w$, and $w$
enters the singular structure of the problem on the same footing as
$\smin$.

\begin{lemma}[exact guided factor]\label{lem:gfactor}
Under Hypothesis~\ref{hyp:commuting}, the exact solution of
\eqref{eq:guided-flow} on an eigen-coordinate over one step $\sigma_n \to
\sigma_{n+1}$ is $\xi_{n+1} = \Phi_w\,\xi_n$ with
\begin{equation}\label{eq:exact-factor}
  \Phi_w
  \;=\;
  \left(\frac{a+\sigma_{n+1}^2}{a+\sigma_n^2}\right)^{\!\frac{1+w}{2}}
  \left(\frac{b+\sigma_n^2}{b+\sigma_{n+1}^2}\right)^{\!\frac{w}{2}} .
\end{equation}
In particular, on a discriminative coordinate ($a=0$), writing $r =
\sigma_{n+1}/\sigma_n$,
\begin{equation}\label{eq:exact-disc}
  \Phi_w \;=\; r^{\,1+w}
  \left(\frac{b+\sigma_n^2}{b+\sigma_{n+1}^2}\right)^{\!\frac{w}{2}},
\end{equation}
which tends to $r^{1+w}$ once $\sigma_n^2 \ll b$.
\end{lemma}

\newcommand{\ProofGFactor}{%
Integrate \eqref{eq:mu}: $\int \mu_w \,\dd\log\sigma = \tfrac{1+w}{2}
\log(a+\sigma^2) - \tfrac{w}{2}\log(b+\sigma^2)$ up to a constant, since
$\dd\log(c+\sigma^2) = 2A_c\,\dd\log\sigma$ with $A_c =
\sigma^2/(c+\sigma^2)$. Exponentiate the increment.
}
\ifarxiv\begin{proof}\ProofGFactor\end{proof}\fi

The singular limit now has two knobs. As $\smin \to 0$ the discriminative
layer contracts by the anomalous power $\sigma^{1+w}$; as $w$ grows at
fixed $\smin$ the same layer stiffens without bound. We ask two questions:
which part of DDIM$+$CFG's failure is the pure terminal-layer solver barrier,
and can that barrier be removed by a one-coefficient fitted repair? Uniform
accuracy through the finite-$\sigma$ crossover is separate, needing either mesh
resolution of the crossover or the spectrum-aware Gaussian factor.

\section{The guided barrier and the fitted operator}\label{sec:guided-program}

This section establishes one claim:
\emph{classifier-free guidance introduces a second singular parameter $w$,
and on discriminative terminal directions the $\sigma$-clock frozen-field
step---DDIM applied to the guided denoiser---is no longer fitted.} The
program has three parts. First, a barrier: the guided analog of the
unguided clock dichotomy~\citep{ap2026}, with sharp thresholds in the log-step
$h$ and weight $w$
(Theorem~\ref{thm:gbarrier}). Second, an operator: a fitted step that
removes the pure discriminative terminal-layer reflection and
$\smin$-divergent residual blow-up at zero additional cost
(Definition~\ref{defn:gfitted}). Third, accuracy: a first-order certificate on
the discriminative crossover (Theorem~\ref{thm:cross-ap}); the whole-model
theory remains future work. Throughout, the
reference solution is the exact guided flow's own pushforward: we charge
the \emph{discretization} of \eqref{eq:guided-flow}, not the distortion of
the continuous guided law away from the tilted conditional
(Section~\ref{sec:intro}).

\paragraph{The frozen step.}
DDIM applied to $\Dw$---equivalently, $\sigma$-clock Euler with the guided
field frozen at the left endpoint---is
\begin{equation}\label{eq:cfg-step}
  x_{n+1} \;=\; x_n + (\sigma_{n+1}-\sigma_n)\,\etaw(x_n,\sigma_n),
\end{equation}
which on an eigen-coordinate of the model contracts by
\begin{equation}\label{eq:Gcfg}
  G_{\rm CFG} \;=\; 1 - (1-r)\bigl[(1+w)A - wB\bigr],
  \qquad r = \frac{\sigma_{n+1}}{\sigma_n} = \eexp^{-h}.
\end{equation}
On the pure discriminative layer $(A,B)=(1,0)$ this is
\begin{equation}\label{eq:Gw}
  G_w(h) \;=\; 1 - (1+w)\,(1-\eexp^{-h}) \;=\; (1+w)\,\eexp^{-h} - w ,
\end{equation}
to be compared with the exact factor $\eexp^{-(1+w)h}$ of
Lemma~\ref{lem:gfactor}. Here $h = \log(\sigma_n/\sigma_{n+1})$ is the step
in $\lambda = -\log\sigma$ (half the VE log-SNR), so a uniform
$\lambda$-mesh has constant $h$. Along a trajectory, call
$\max_n \|\etaw(x_n,\sigma_n)\|/\|\etaw(x_0,\sigma_0)\|$ the
\emph{residual-amplification certificate}, written $\mathrm{amp}$; the
sampler already forms every quantity in it, and a scheme is residual-stable if
the certificate stays bounded as $\smin \to 0$.
Everything in Theorem~\ref{thm:gbarrier} is a statement about the
elementary function \eqref{eq:Gw}.

\begin{theorem}[guided clock barrier]\label{thm:gbarrier}
Consider the frozen step \eqref{eq:cfg-step} on the pure discriminative
layer, $\etaw = (1+w)\,\xi/\sigma$ with $w > 0$, on a uniform
$\lambda$-mesh of step $h$. Let
\begin{equation}\label{eq:thresholds}
  h_\flat \;=\; \log\!\Bigl(1+\tfrac1w\Bigr),
  \qquad
  h_\sharp \;=\; \log\!\Bigl(1+\tfrac2w\Bigr),
  \qquad
  h_\infty \;=\; \log\!\Bigl(\tfrac{1+w}{\,w-1\,}\Bigr)\ (w>1).
\end{equation}
Then $h_\flat < h_\sharp < h_\infty$, and:
\begin{itemize}
\item[(a)] \emph{Sign.} $G_w(h) > 0$ iff $h < h_\flat$; equivalently, at
fixed $h$, sign preservation fails once $w > 1/(\eexp^{h}-1)$. For $h >
h_\flat$ every step reflects the coordinate across the class manifold.
\item[(b)] \emph{Residual amplification.} The per-step residual
amplification is $|G_w(h)|\,\eexp^{h}$, which is $\le 1$ iff $h \le
h_\sharp$; equivalently, raw amplification begins once $w >
2/(\eexp^{h}-1)$. For $h > h_\sharp$ the amplification factor is
$w\eexp^{h} - (1+w) > 1$ per step. On a terminal subwindow
$\sigma \in [\smin, \varepsilon\sqrt b\,]$, where $B \le \varepsilon^2/(1+\varepsilon^2)$
makes the exact per-step factor pure-layer up to $O(\varepsilon^2)$, of
$\lambda$-length $\Lambda_\varepsilon = \log(\varepsilon\sqrt b/\smin)$ the guided
residual grows by
\begin{equation}\label{eq:blowup}
  \bigl(w\eexp^{h}-(1+w)\bigr)^{\Lambda_\varepsilon/h}
  \;=\;
  \Bigl(\frac{\varepsilon\sqrt b}{\smin}\Bigr)^{\!\rho},
  \qquad
  \rho = \frac{\log\bigl(w\eexp^{h}-(1+w)\bigr)}{h} > 0 ,
\end{equation}
so for fixed $\varepsilon>0$ and $h > h_\sharp$ the scheme fails the amplification
certificate as $\smin \to 0$.
\item[(c)] \emph{Absolute stability.} $|G_w(h)| \le 1$ for all $h$ when $w
\le 1$; for $w > 1$, iff $h \le h_\infty$.
\end{itemize}
All three thresholds decay like $c/w$ as $w \to \infty$, with $c = 1, 2,
2$ respectively.
\end{theorem}

\newcommand{\ProofGBarrier}{%
From \eqref{eq:Gw}: $G_w > 0 \iff \eexp^{-h} > w/(1+w) \iff h <
\log(1+1/w)$, which is (a). For (b), on the pure layer $\etaw \propto
\xi/\sigma$, so the residual ratio per step is $|\xi_{n+1}/\xi_n| \cdot
\sigma_n/\sigma_{n+1} = |G_w|\eexp^{h}$. On the branch $G_w \le 0$ this is
$(w - (1+w)\eexp^{-h})\eexp^{h} = w\eexp^{h} - (1+w)$, which exceeds $1$
iff $\eexp^{h} > 1 + 2/w$; on the branch $G_w \ge 0$ it equals $(1+w) -
w\eexp^{h} \le 1$ always. Compounding over $\Lambda_\varepsilon/h$ steps gives
\eqref{eq:blowup}. For (c), $G_w \le 1$ always, and $G_w \ge -1 \iff
(1+w)\eexp^{-h} \ge w-1$, vacuous for $w \le 1$ and equivalent to $h \le
h_\infty$ otherwise. The ordering $h_\sharp < h_\infty$ follows from
$\frac{1+w}{w-1} - \bigl(1+\frac2w\bigr) = \frac{2}{w(w-1)} > 0$.
}
\ifarxiv\begin{proof}\ProofGBarrier\end{proof}\fi

In words: sign preservation fails first (a, above $h_\flat$: every step
reflects across the class manifold), the certificate second (b, above
$h_\sharp$: the residual compounds, diverging as $\smin\to0$), absolute
stability last (c, only for $w>1$).

\begin{corollary}[two-tier guidance tax]\label{cor:tax}
On a uniform $\lambda$-mesh covering the discriminative window (of
$\lambda$-length $\Lambda_b = \log(\sqrt b/\smin)$), avoiding
reflection requires
\begin{equation}\label{eq:tax}
  N \;\ge\; \frac{\Lambda_b}{\log(1+1/w)} \;\sim\; w\,\Lambda_b ,
\end{equation}
while merely keeping the residual nonexpansive requires $N \ge \Lambda_b/\log(1+2/w)
\sim \tfrac{w}{2}\Lambda_b$. A uniform \emph{global} mesh must meet this bound
throughout the window, so $\Lambda_b$ is replaced by the full horizon
$\Lambda=\log(\smax/\smin)$: at working scale $g = 7.5$
($\Lambda = \log(80/0.002) \approx 10.6$), reflection-free sampling already demands
$N \gtrsim 74$ steps---an order of magnitude above unguided-DDIM-accurate
budgets \citep{karras2022edm}.
\end{corollary}

\begin{remark}[the reflecting window]\label{rem:window}
In the window $h \in (h_\flat, h_\sharp]$ the scheme is \emph{reflecting
but nonexpansive}: the flip does not grow the residual. The flip is
the mechanism of visible artifacts---the iterate lands on the wrong side
of the class manifold each step---while the amplification crossing at $h_\sharp$ is
where the certificate, and any uniform accuracy claim, actually fails.
Practitioners tune $w$ down at low $\sigma$ well before the certificate
diverges: part (a), not part (b), is the threshold their tuning discovers.
\end{remark}

\begin{remark}[the one-jump limit is oversaturation]\label{rem:onejump}
The following identity is model-free. A single frozen step
\eqref{eq:cfg-step} from $\sigma_n$ to $\sigma_{n+1} = r\sigma_n$ with $r
\to 0$ gives
\begin{equation}\label{eq:onejump}
  x_{n+1} \;\to\; x_n - \sigma_n\,\etaw(x_n,\sigma_n)
  \;=\; (1+w)\,\Dc - w\,\Du \;=\; \Dw(x_n,\sigma_n):
\end{equation}
one-step DDIM$+$CFG lands on the \emph{extrapolated} denoiser, displaced
from the class manifold by $-w(\Du-\Dc)$. On the model this is the
reflection $\xi \mapsto -w\,\xi$ of Theorem~\ref{thm:gbarrier}(a) pushed to
its endpoint. We read this as the discrete mechanism of the oversaturation
and norm-inflation phenomenology reported for large guidance scales
\citep{sadat2025apg}: it is a property of the \emph{solver} at coarse steps, not
of the continuous guided flow---on the calibration model the exact flow
contracts to $\Dc$ as $r^{1+w}$ while the frozen step overshoots to $\Dw$.
\end{remark}

\paragraph{The guided fitted step.}
The repair is dictated by the same logic that produced DDIM in the
unguided setting \citep{ap2026}: integrate the stiff layer exactly and
freeze what is slow. The
gauge $\tilde\xi = \sigma^{-w}\xi$ rectifies the anomalous exponent $1+w$
to exponent $1$; transporting the $\sigma$-clock Euler step through the
gauge and back yields, on the whole state,

\begin{definition}[guided fitted step]\label{defn:gfitted}
With $r = \sigma_{n+1}/\sigma_n$ and all fields evaluated at
$(x_n,\sigma_n)$,
\begin{equation}\label{eq:g2-denoiser}
  x_{n+1}
  \;=\;
  \Dc \;+\; r\,\bigl(x_n - \Dc\bigr)
  \;+\; \bigl(r^{\,1+w} - r\bigr)\,\bigl(\Du - \Dc\bigr) .
\end{equation}
equivalently, in residual form,
\begin{equation}\label{eq:g2-residual}
  x_{n+1}
  \;=\;
  x_n + (\sigma_{n+1}-\sigma_n)\,\etac
  \;+\; \sigma_n\,\bigl(r^{\,1+w}-r\bigr)\,\gamma .
\end{equation}
\end{definition}

The scheme touches exactly one coefficient of classifier-free guidance:
vanilla DDIM$+$CFG is \eqref{eq:g2-denoiser} with the coefficient
$w(r-1)$ in place of $r^{1+w}-r$ on the guidance direction $\Du - \Dc$. It
uses the same two denoiser evaluations per step, so the cost---and the
implementation delta---is one line.
Proposition~\ref{prop:gfit-sign} verifies the factor coordinate-wise;
Lemma~\ref{lem:g2-unique} pins the coefficient from terminal exactness alone.

\begin{proposition}[consistency]\label{prop:gconsist}
$r^{1+w} - r = w(r-1) + \tfrac12 w(w+1)\,h^2 + O(h^3)$ as $h \to 0$, so
\eqref{eq:g2-denoiser} is a first-order-consistent discretization of the
guided flow \eqref{eq:guided-flow}, agreeing with DDIM$+$CFG through
$O(h)$; and at $w = 0$ both coefficients vanish identically, collapsing the
scheme to DDIM.
\end{proposition}

\newcommand{\ProofGConsist}{%
$r^{1+w} - r = \eexp^{-(1+w)h} - \eexp^{-h}$ and $w(r-1) = w(\eexp^{-h}-1)$;
both are $-wh + O(h^2)$, and expanding to second order,
$r^{1+w}-r-w(r-1) = \eexp^{-(1+w)h}-(1+w)\eexp^{-h}+w = \tfrac12 w(w+1)h^2 + O(h^3)$.
}
\ifarxiv\begin{proof}\ProofGConsist\end{proof}\fi

\begin{proposition}[the fitted step is sign-preserving and terminal-exact]\label{prop:gfit-sign}
Under Hypothesis~\ref{hyp:commuting}, the step \eqref{eq:g2-denoiser}
contracts an eigen-coordinate by
\begin{equation}\label{eq:gfit-decomp}
  G_{\rm fit}
  \;=\; 1 - (1-r)A + \bigl(r^{1+w}-r\bigr)(A-B)
  \;=\; (1-A) \;+\; r\,B \;+\; r^{\,1+w}\,(A-B) ,
\end{equation}
a sum of nonnegative terms; hence $G_{\rm fit} > 0$ on every coordinate,
for every mesh, and every $w \ge 0$: the fitted step preserves the sign of
every analyzed coordinate.
It is exact in the three terminal regimes---$G_{\rm fit} = r$ on
shared normals $(A=B=1)$, $G_{\rm fit} = r^{1+w}$ on pure discriminative
normals $(A=1, B=0)$, and $G_{\rm fit} = 1$ on frozen tangential directions
$(A=B=0)$---and its one-jump limit is $x_{n+1} \to \Dc$, the conditional
denoiser (the terminal-limit class projection), in contrast with
\eqref{eq:onejump}.
\end{proposition}

\newcommand{\ProofGFitSign}{%
On an eigen-coordinate, $\Dc = (1-A)\xi$, $x-\Dc = A\xi$, and $\Du - \Dc =
(A-B)\xi$; substituting into \eqref{eq:g2-denoiser} gives the first
expression in \eqref{eq:gfit-decomp}, and regrouping $1 - A + rA - r(A-B) +
r^{1+w}(A-B)$ gives the second. Nonnegativity of each term is
$0 \le B \le A \le 1$ and $0 < r < 1$. The regime values and the $r \to 0$
limit are read off termwise.
}
\ifarxiv\begin{proof}\ProofGFitSign\end{proof}\fi

\begin{lemma}[uniqueness of the scalar terminal coefficient]\label{lem:g2-unique}
Consider scalar coefficient updates of the form
\[
  x_{n+1} = \Dc + r(x_n-\Dc) + \alpha(r,w)(\Du-\Dc),
\]
where $\alpha$ depends only on $(r,w)$.  The unique such coefficient that is
exact on the pure discriminative terminal regime $(A,B)=(1,0)$ is
\[
  \alpha(r,w) = r^{1+w}-r .
\]
Consequently, terminal exactness, DDIM recovery at $w=0$, and exactness on
the shared-normal and tangent terminal regimes determine the fitted coefficient
without any spectrum information.
\end{lemma}

\newcommand{\ProofGTwoUnique}{%
On shared normals $(A,B)=(1,1)$ one has $\Du-\Dc=0$, so every scalar
coefficient gives the DDIM factor $r$.  On tangent directions
$(A,B)=(0,0)$, both denoisers equal the state and every coefficient gives
factor $1$.  On a pure discriminative normal $(A,B)=(1,0)$, the factor of
the scalar update is $r+\alpha(r,w)$.  Exactness requires
$r+\alpha(r,w)=r^{1+w}$, hence the displayed coefficient, uniquely.
}
\ifarxiv\begin{proof}\ProofGTwoUnique\end{proof}\fi

\begin{proposition}[finite crossover tax for the fitted step]\label{prop:g2-crossover-tax}
On a single discriminative Gaussian coordinate with $a=0<b$, define
\[
  q(\sigma)=1-B(\sigma)=\frac{b}{b+\sigma^2},\qquad
  \mu(\sigma)=\mu_w(\sigma)=1+wq(\sigma).
\]
Let $\sigma_0>\sigma_1>\cdots>\sigma_K$ be any decreasing mesh,
$r_n=\sigma_{n+1}/\sigma_n$, $q_n=q(\sigma_n)$, and
$\mu_n=\mu(\sigma_n)$.  For the fitted step \eqref{eq:g2-denoiser}, the
guided residual ratio over one step is
\begin{equation}\label{eq:g2-crossover-ratio}
  \frac{|\etaw(\xi_{n+1},\sigma_{n+1})|}
       {|\etaw(\xi_n,\sigma_n)|}
  =
  \frac{\mu_{n+1}}{\mu_n}
  \bigl[(1-q_n)+r_n^{\,w}q_n\bigr].
\end{equation}
Hence for every prefix $k$,
\begin{equation}\label{eq:g2-finite-tax}
  \frac{|\etaw(\xi_k,\sigma_k)|}
       {|\etaw(\xi_0,\sigma_0)|}
  \le
  \frac{\mu_k}{\mu_0}
  \le 1+w .
\end{equation}
Thus the fitted step removes the $\smin$-divergent terminal blow-up of
Theorem~\ref{thm:gbarrier} on the discriminative crossover, but it may still
pay a finite $O(1+w)$ crossover tax on a coarse mesh.
\end{proposition}

\newcommand{\ProofGTwoTax}{%
For $a=0$, $A\equiv1$ and $q=1-B$.  The decomposition
\eqref{eq:gfit-decomp} gives the coordinate factor
\[
  G_{\rm fit}
  = r_nB_n+r_n^{1+w}(1-B_n)
  = r_n\bigl[(1-q_n)+r_n^{\,w}q_n\bigr].
\]
Since $\etaw=\mu(\sigma)\xi/\sigma$, the residual ratio is
$(\mu_{n+1}/\mu_n)(G_{\rm fit}/r_n)$, which is
\eqref{eq:g2-crossover-ratio}.  The bracket is at most $1$ because
$0\le q_n\le1$ and $0<r_n^w\le1$, so products telescope:
\[
  \prod_{n=0}^{k-1}
  \frac{|\etaw(\xi_{n+1},\sigma_{n+1})|}
       {|\etaw(\xi_n,\sigma_n)|}
  \le
  \prod_{n=0}^{k-1}\frac{\mu_{n+1}}{\mu_n}
  =
  \frac{\mu_k}{\mu_0}
  \le 1+w .
\]
}
\ifarxiv\begin{proof}\ProofGTwoTax\end{proof}\fi

Bounded amplification in fact upgrades to accuracy: on this crossover the
accumulated log-defect (uniform mesh, $h\le1$) obeys
$0\le\log\Theta_K\le\tfrac{1+h}{4}\,w(w+2)\,h$, so the
fitted step is \emph{first-order uniformly accurate} against the exact guided
flow's own pushforward, with a $W_2$ bound whose constant is independent of $b$,
$K$, and $\smin$ (the $O(w^2)$ factor is not uniform in $w$).  We defer the
formal statement and proof to Theorem~\ref{thm:cross-ap} in
Appendix~\ref{app:proofs}.

\newcommand{\ProofGCrossAP}{%
On $a=0$ the exact one-step factor (Lemma~\ref{lem:gfactor}) is
$\Phi_n=r^{1+w}\bigl((1+u_n)/(1+r^2u_n)\bigr)^{w/2}$, with $u_n=\sigma_n^2/b$,
$r=\eexp^{-h}$; the fitted factor (Proposition~\ref{prop:gfit-sign}, $A=1$) is
$G_n=\eexp^{-h}(u_n+\eexp^{-wh})/(1+u_n)$.  Writing $y_n=\log u_n$,
$\delta(h,y):=\log(G_n/\Phi_n)
= wh+\log(\eexp^{y}+\eexp^{-wh})-\bigl(1+\tfrac w2\bigr)\log(1+\eexp^{y})
+\tfrac w2\log(1+\eexp^{y-2h})$.
Then $\delta(0,y)=\partial_h\delta(0,y)=0$ (first-order consistency,
Proposition~\ref{prop:gconsist}), and with $F(y)=\eexp^{y}/(1+\eexp^{y})^2$
one has $\partial_h^2\delta=w^2F(y+wh)+2wF(y-2h)\ge0$; hence $\delta\ge0$, so
$\Theta_K=\prod_n\eexp^{\delta(h,y_n)}\ge1$, and Taylor's formula gives
$\delta(h,y_n)=\int_0^h(h-s)\bigl[w^2F(y_n+ws)+2wF(y_n-2s)\bigr]\,ds$.
On the uniform mesh $y_{n+1}=y_n-2h$; since $\int F=1$ and
$\mathrm{TV}(F)=\tfrac12$, the rectangle bound gives $2h\sum_n F(y_n+c)\le1+h$
for every shift $c$, so
$\log\Theta_K\le\int_0^h(h-s)\tfrac{1+h}{2h}(w^2+2w)\,ds=\tfrac{1+h}{4}w(w+2)h$.
Synchronous coupling with $|\xi_K^{\rm ex}|\le|\xi_0|$ (the exact
discriminative factor contracts) gives the $W_2$ bound.
}

Extending this crossover accuracy to the tangential layers and learned
geometry needs the spectrum-aware Gaussian factor \eqref{eq:exact-factor}
and is future work.

\begin{remark}[the terminal side of the guidance interval]\label{rem:interval}
Limited-interval guidance \citep{kynkaanniemi2024interval} disables
guidance outside a tuned band $(\sigma_{\rm lo}, \sigma_{\rm hi})$ and
leaves open whether the band can be derived rather than tuned.
Theorem~\ref{thm:gbarrier} supplies the terminal side without free
parameters: on schedules whose local step $h_n$ grows toward low $\sigma$
(e.g.\ Karras power meshes; on a uniform $\lambda$-mesh $h_n$ is constant and
there is no distinct crossing), the level at which $h_n$ crosses $h_\flat(w)
= \log(1+1/w)$ predicts $\sigma_{\rm lo}$: below it, frozen-field guidance at weight $w$ is
unaffordable and the interval or the fitted step must take over. The \emph{high}-noise cutoff
$\sigma_{\rm hi}$ is not explained by the barrier: there $\|\Du - \Dc\|$ is
small against the model's own error, a signal-level criterion outside this
paper's ledger.
\end{remark}

\newcommand{\ReverseFull}{%
If a reversed singular direction has $b = 0 < a$---the marginal collapsed
where the class has not---then the exponent \eqref{eq:mu} tends to $-w$,
negative for every $w > 0$: the guided flow itself expands the coordinate
as $\sigma \to 0$, so no discretization can be uniformly stable because the
continuous target is not.  More generally, reverse ordering with
$0 < b < a$ can create finite-$\sigma$ expansion windows, but only the
reversed singular case creates this terminal instability.  This is a
property of the guidance target, not of the solver; the class-subset
hypothesis excludes it by assumption, and autoguidance-style constructions
aim to enforce the same directionality.
}
\ifarxiv
\begin{remark}[the reverse ordering]\label{rem:reverse}
\ReverseFull
\end{remark}
\else
\begin{remark}[the reverse ordering]\label{rem:reverse}
If a direction has $b = 0 < a$, the exponent \eqref{eq:mu} tends to $-w$:
the \emph{continuous} guided flow expands as $\sigma \to 0$, an instability
of the guidance target that no solver can repair; the class-subset
hypothesis excludes it by assumption
(Appendix~\ref{app:reverse}).
\end{remark}
\fi

\paragraph{Certificates.}
The certificate is removal of the $\smin$-divergent pure-layer mechanism
(Proposition~\ref{prop:g2-crossover-tax}) plus first-order accuracy on the
discriminative crossover (Theorem~\ref{thm:cross-ap}), not global residual
nonexpansion through every finite-$\sigma$ crossover.

%% file: guided_empirical_certificates.tex
%

\section{Empirical diagnostics on learned checkpoints}
\label{sec:checkpoint-certs}

\paragraph{Synthetic crossover.}
Our primary residual and image-quality comparisons fix their schemes, cells, and
metrics before each run; the feature-space, latent-diffusion, and dense-reference
results are exploratory diagnostics.  The single-coordinate Gaussian diagnostic isolates the theorem's
singular mechanism: on uniform log-$\sigma$ meshes with $a=0<b$, CFG
develops the coarse-mesh branch of Theorem~\ref{thm:gbarrier} while the
fitted coefficient stays below Proposition~\ref{prop:g2-crossover-tax}'s
finite-tax bound (Figure~\ref{fig:crossover}).  In the default $w$--$N$ sweep,
worst-case amplification reaches $10^3$--$10^9$ for CFG but stays $\le4.66$ for
the fitted step, below each $1+w$ bound---a theorem figure, not a
learned-checkpoint quality metric.

\begin{figure}[t]
\centering
\includegraphics[width=0.43\linewidth]{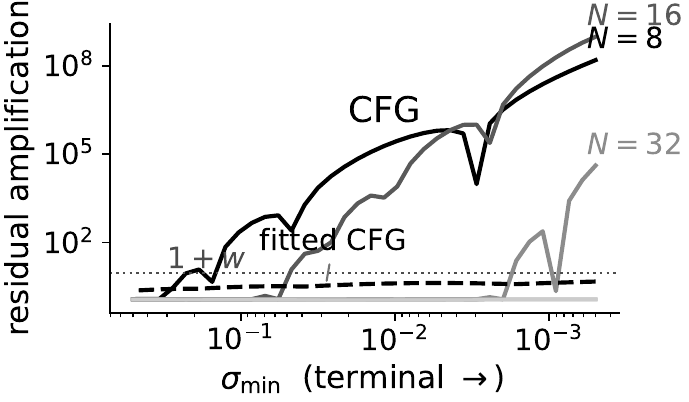}
\caption{Synthetic crossover at $w=8$: residual amplification as
$\smin\to0$.  CFG develops the $\smin$-divergent branch of
Theorem~\ref{thm:gbarrier}(b) on coarse meshes ($N\in\{8,16,32\}$), stabilizing
only once $h<h_\sharp$---the flat gray curve ($N{=}64$) coincides with the exact
flow.  The fitted step (dashed) stays below the $1+w$ tax on every mesh.}
\label{fig:crossover}
\end{figure}

\paragraph{Real-checkpoint residual diagnostics.}
We then test the coefficient as a drop-in replacement on NVIDIA \textsc{edm}
CIFAR-10 VP checkpoints \citep{karras2022edm} (the `VP' label is network
preconditioning only; sampling uses the $\sigma$-parameterization): $\Dc$ is the
class-conditional checkpoint, $\Du$ the separately trained unconditional one, and
both schemes use the same Karras mesh, seeds, labels, and NFE.  The grid is $w\in\{4,6.5,8\}$ and
$N\in\{8,16,32\}$, with two 64-seed blocks (class $0$, seeds from $1000$;
class $1$, from $2000$).
Table~\ref{tab:real-checkpoint-gate} reports paired fitted/CFG ratios for
the amplification median and tail, and paired clipping differences; lower is
better.

\begin{table}[t]
\centering
\caption{Real-checkpoint residual diagnostics.  amp med and amp p95 (median and 95th-percentile
residual amplification) are fitted-CFG/CFG ratios ($<1$ means the repair shrinks
it); clip $\Delta$ is fitted CFG minus CFG for the median fraction of
final-denoised pixels outside [-1,1].}
\label{tab:real-checkpoint-gate}
{\scriptsize
\IfFileExists{manuscript/real_gate_artifacts/real_checkpoint_gate_table.tex}{%
  \input{manuscript/real_gate_artifacts/real_checkpoint_gate_table.tex}}{%
\input{real_gate_artifacts/real_checkpoint_gate_table.tex}}}
\end{table}

In both blocks, for every tested
$N\in\{8,16,32\}$ at $w=6.5$ and $w=8$, fitted CFG improves the amp median,
amp p95, and clipping median.  The hardest cell is the
clearest: at $w=8,N=8$, the class-$0$ amp median/p95 falls from
$4.57/13.35$ under CFG to $1.41/1.71$ under fitted CFG, while the clipping median drops from
$0.379$ to $0$; the class-$1$ replicate matches (amp p95 ratio $0.116$,
clipping delta $-0.23$).  At $w=4$, where the
theory predicts a much weaker tax, the median amplification is mixed, but
the amp tail and saturation proxy still improve in most cells.  We use this
as evidence for high-guidance stabilization, not uniform superiority at
weak guidance.

\paragraph{Image-quality evaluation.}
We next ran two 5000-image balanced-class evaluations on the hard cell $w=8,N=8$,
plus a 50k replication, each including a terminal-interval baseline (vanilla CFG
correction off when $h>h_\flat$, after \citet{kynkaanniemi2024interval}; on this
coarse mesh every step crosses $h_\flat$, so it reduces to unguided conditional
DDIM).  We score with FID and KID (feature-space quality metrics; lower is
better).
Table~\ref{tab:image-quality} collects the results.  Fitted CFG has the best
FID in all three blocks at zero extra NFE, and the certificates replicate
(amp p95 falls from $13{+}$ to ${\approx}1.6$, clipping p95 from ${\approx}0.57$
to $0$).  The metrics disagree, however: KID is lowest for CFG in every block.
We therefore read this as support for high-guidance stabilization and
FID improvement, not an unqualified image-quality win; feature-manifold and
sharpness diagnostics are consistent with CFG's oversaturation contributing to
the split (Figure~\ref{fig:hardcell-samples}), and a DINOv2
backbone-swapped audit gives the same fitted-favoring ordering, so the split is
not Inception-specific (Appendices~\ref{app:split} and~\ref{app:dino}).

\begin{table}[t]
\centering
\caption{Image-quality evaluation at the hard cell $w{=}8,N{=}8$: two 5k seed blocks
(A, B) and a 50k replication, for CFG, the fitted repair, and the
terminal-interval baseline.  FID/KID are feature-space quality metrics; amp p95
and clip p95 are the 95th-percentile residual amplification and final-denoise
clipping fraction (per scheme, not a ratio); target acc is classifier
target-label accuracy (\%).  \textbf{Bold} marks the per-block best of the two
feature metrics: fitted CFG wins FID everywhere, CFG wins KID everywhere.}
\label{tab:image-quality}
{\small
\begin{tabular}{llccccc}
\toprule
 & scheme & FID $\downarrow$ & KID $\downarrow$ & amp p95 $\downarrow$ & clip p95 $\downarrow$ & target acc $\uparrow$ \\
\midrule
\multirow{3}{*}{5k (A)} & CFG      & $32.44$ & $\mathbf{0.0110}$ & $13.32$ & $0.570$ & $94.96$ \\
                        & fitted   & $\mathbf{25.22}$ & $0.0181$ & $1.66$ & $0$ & $94.66$ \\
                        & interval & $26.46$ & $0.0224$ & $1.63$ & $0$ & $89.24$ \\
\midrule
\multirow{3}{*}{5k (B)} & CFG      & $31.98$ & $\mathbf{0.0107}$ & $13.20$ & $0.572$ & $94.88$ \\
                        & fitted   & $\mathbf{24.89}$ & $0.0173$ & $1.64$ & $0$ & $94.34$ \\
                        & interval & $26.04$ & $0.0211$ & $1.62$ & $0$ & $90.00$ \\
\midrule
\multirow{3}{*}{50k}    & CFG      & $27.86$ & $\mathbf{0.0108}$ & $13.34$ & $0.569$ & $95.14$ \\
                        & fitted   & $\mathbf{20.88}$ & $0.0177$ & $1.65$ & $0$ & $94.83$ \\
                        & interval & $22.08$ & $0.0221$ & $1.63$ & $0$ & $88.95$ \\
\bottomrule
\end{tabular}}
\end{table}

\begin{figure}[t]
\centering
\includegraphics[width=0.46\linewidth]{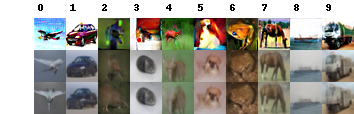}
\caption{Uncurated hard-cell samples ($w{=}8$, $N{=}8$; first ten CIFAR-10
classes, matched seeds).  Rows: CFG, fitted CFG, terminal-interval.  CFG's row
is visibly more saturated and higher-contrast; the repair and interval are
lower-contrast at zero extra NFE.  Qualitative only---KID still favors CFG.}
\label{fig:hardcell-samples}
\end{figure}

Finally, the target-accuracy column of Table~\ref{tab:image-quality} is a
classifier-conditionality diagnostic (classifier at $94.37\%$ on the CIFAR-10 test
set): in every block fitted CFG stays within $0.6$pp of CFG yet ${\sim}5$pp
above interval, keeping CFG's conditional pull without its clipping pathology.

\paragraph{Robustness grid.}
The 50k replication (Table~\ref{tab:image-quality}) reduces the 5k-sampling-artifact
concern for the FID ordering, and the KID split persists at scale.  A nine-cell grid
($w\in\{4,6.5,8\}$, $N\in\{8,16,32\}$, 5k images each) shows it is not
cell-specific: fitted CFG improves FID in $9/9$ cells, clipping p95 in $9/9$,
and amp p95 in $8/9$ (exception $w{=}4,N{=}32$, both near $1$).  The
terminal-interval baseline wins FID at $N\in\{16,32\}$ and KID at $N=32$ but
loses conditionality everywhere ($3.24$--$5.72$pp).

\paragraph{Interval-baseline separation.}
Is the fitted step just guidance shutdown?  We compared it against a
\emph{tuned} conditionality-preserving limited interval (selected per cell) and
the parameter-free terminal interval (Appendix~\ref{app:interval}).  In the
tested cells, the tuned interval preserves target accuracy but does not beat the
fitted step on FID; the terminal interval can win FID (e.g.\ $12.63$ vs $17.12$
at $w{=}8,N{=}16$) but drops target accuracy by $5$--$6$pp.  So the fitted step gives interval-like
residual stabilization while keeping the guided update active---it is not
explained away by turning guidance off.

\paragraph{All-class audit.}
Extending the residual diagnostic from two classes to all ten CIFAR-10 classes on the
high-guidance cells (60 class--cell pairs), clipping/saturation is non-increased
in $60/60$ pairs and amp p95 improves in $54/60$: global saturation robustness
and near-global residual-tail improvement (Appendix~\ref{app:allclass}).

\paragraph{Latent-diffusion smoke.}
The mechanism is not specific to pixel-space CIFAR.  On Stable Diffusion~1.5
DDIM at high guidance, the same coefficient sharply reduces pixel saturation
(p95 $0.134\to0.009$ at guidance $12$, over 256 images) while preserving CLIP
image--text alignment ($0.320\to0.314$); terminal guidance shutdown reaches
lower saturation but collapses CLIP to $0.271$ (Appendix~\ref{app:ldm}).  This
is a cross-domain smoke test, not a Stable Diffusion benchmark.

%% file: real_gate_artifacts/real_checkpoint_gate_table.tex
\begin{tabular}{ccrrrrrr}
\toprule
$w$ & $N$ & \multicolumn{3}{c}{class0} & \multicolumn{3}{c}{class1} \\
 &  & amp med $\downarrow$ & amp p95 $\downarrow$ & clip $\Delta\downarrow$ & amp med $\downarrow$ & amp p95 $\downarrow$ & clip $\Delta\downarrow$ \\
\midrule
4.0 & 8 & 0.993 & 0.867 & -0.0236 & 0.997 & 0.985 & -0.0177 \\
4.0 & 16 & 1.001 & 0.878 & -0.0225 & 1.001 & 1.012 & -0.0360 \\
4.0 & 32 & 1.013 & 0.871 & -0.0231 & 1.002 & 1.012 & -0.0529 \\
6.5 & 8 & 0.695 & 0.149 & -0.2889 & 0.666 & 0.132 & -0.1421 \\
6.5 & 16 & 0.821 & 0.534 & -0.2926 & 0.933 & 0.659 & -0.1978 \\
6.5 & 32 & 0.836 & 0.512 & -0.3120 & 0.971 & 0.668 & -0.2231 \\
8.0 & 8 & 0.309 & 0.128 & -0.3786 & 0.203 & 0.116 & -0.2306 \\
8.0 & 16 & 0.668 & 0.132 & -0.3937 & 0.593 & 0.129 & -0.2876 \\
8.0 & 32 & 0.544 & 0.216 & -0.4271 & 0.704 & 0.288 & -0.3853 \\
\bottomrule
\end{tabular}

%% file: sections/045_related_work.tex
\section{Related work}
\label{sec:related}

The CFG-repair family is crowded, and methods differ mainly in their
\emph{mathematical object}: a manifold projection, a guidance-vector rescaling,
an interval schedule, a dynamic weight, or---ours---a terminal fitted-operator
coefficient.  In brief: CFG++ \citep{chung2025cfgpp} constrains the
update to the data manifold; limited-interval guidance
\citep{kynkaanniemi2024interval} switches guidance off outside a $\sigma$
band (Remark~\ref{rem:interval} derives its terminal edge parameter-free);
APG \citep{sadat2025apg} down-weights the update's parallel component;
CFG-Zero$^\star$ \citep{fan2025cfgzero} optimizes the scale and zero-inits
early steps; and C$^2$FG \citep{gao2026c2fg} decays the weight in time.
None derives the terminal-fitted coefficient
$r^{1+w}-r$ from the guided exponent---where the barrier lives.
An orthogonal line analyses what the \emph{continuous} guided law samples---CFG
is a predictor--corrector, not a sampler of its tilted target
\citep{bradley2024cfg,chidambaram2024cfg}---the ``model'' bill complementary to
ours.

%% file: sections/05_discussion.tex
\section{Discussion and limitations}
\label{sec:discussion}

\paragraph{What is proved.}
On the commuting class-subset Gaussian model, the barrier
(Theorem~\ref{thm:gbarrier}) is sharp; the fitted step is consistent,
sign-preserving, terminal-exact, pays only a finite $O(1+w)$ crossover tax
(Prop.~\ref{prop:g2-crossover-tax}), and is first-order accurate on the
crossover (Theorem~\ref{thm:cross-ap})---all about the \emph{discretization}
against the guided flow's own pushforward, AP in $\smin$ but not uniform in
$w$.

\paragraph{What is evidence.}
On learned CIFAR-10 \textsc{edm} checkpoints, fitted CFG is a drop-in,
zero-extra-NFE high-guidance stabilizer: the certificates improve residual
amplification and clipping across every $w\in\{6.5,8\}$ cell; the hard-cell
evaluations (5k twice, 50k once) give it the best FID with target accuracy
preserved, and a DINOv2 backbone-swap the same ordering; the nine-cell grid
gives $9/9$ FID wins over CFG.  It separates from tuned and terminal interval
baselines (which either lose target accuracy or fail to beat it on FID), holds
across all ten classes ($60/60$ saturation non-increase, $54/60$ amp-tail), and
transfers to Stable Diffusion~1.5 DDIM as a saturation-reducing coefficient that
keeps CLIP alignment far better than guidance shutdown
(Section~\ref{sec:checkpoint-certs}, Appendices~\ref{app:interval}--\ref{app:allclass}).

\paragraph{What is neither.}
We do not claim (i) an unqualified image-quality win---KID stays lowest for
CFG at 5k and 50k, a split tied to CFG's oversaturation without proven
causation; (ii) that CIFAR verifies the sharp threshold law---the learned
diagnostics are only \emph{consistent with} the model ordering; or (iii) a 50k
KID win or many-cell large-sample result---the 50k run is one hard cell.  Nor
is the fitted step a uniformly better integrator of the vanilla CFG field on
learned checkpoints: against a $512$-step CFG reference its residual tails shrink
but its final-image $L_2$ does not (Appendix~\ref{app:denseref}), an interval can
win FID in some cells by weakening guidance, and the Stable Diffusion result is a
smoke test, not a benchmark.  The class-subset hypothesis is itself a boundary:
where the class is \emph{wider} than the marginal, the continuous guided flow is
unstable (exponent $-w$), which no solver can repair (Remark~\ref{rem:reverse}).

\paragraph{Scope and outlook.}
Every certificate uses only sampler-computed quantities from two separate
checkpoints, so shared-head correlations cannot inflate them.
`Asymptotic-preserving' is used in the calibration-model sense: on learned
checkpoints we audit AP-relevant residuals rather than certify AP.  Next:
shared-dropout CFG, broader latent-diffusion benchmarks, and a whole-model
accuracy theory.

%% file: sections/06_reproducibility.tex
\section*{Reproducibility Statement}
All theoretical results are stated with explicit constants on the model of
Hypothesis~\ref{hyp:commuting}, and complete proofs appear in
Appendix~\ref{app:proofs}.  Every empirical diagnostic in
Section~\ref{sec:checkpoint-certs} is produced by a self-contained script
operating on the public NVIDIA \textsc{edm} CIFAR-10 VP checkpoints with
stated seeds, meshes, sample counts, schemes, and metrics; the
synthetic-crossover and checkpoint-certificate scripts additionally ship CPU
\texttt{--self-test} modes, and the synthetic crossover figure requires no
checkpoints.  The supplementary package contains the scripts, the exact
commands, the run manifests, and all summary artifacts (CSV/JSON) behind every
number reported in the paper, including the 50k and nine-cell grid diagnostics.

%% file: sections/99_appendix.tex
\appendix

\section{Deferred proofs}
\label{app:proofs}

\begin{proof}[Proof of Proposition~\ref{prop:glayers}]
\ProofGLayers
\end{proof}

\begin{proof}[Proof of Lemma~\ref{lem:gfactor}]
\ProofGFactor
\end{proof}

\begin{proof}[Proof of Theorem~\ref{thm:gbarrier}]
\ProofGBarrier
\end{proof}

\begin{proof}[Proof of Proposition~\ref{prop:gconsist}]
\ProofGConsist
\end{proof}

\begin{proof}[Proof of Proposition~\ref{prop:gfit-sign}]
\ProofGFitSign
\end{proof}

\begin{proof}[Proof of Lemma~\ref{lem:g2-unique}]
\ProofGTwoUnique
\end{proof}

\begin{proof}[Proof of Proposition~\ref{prop:g2-crossover-tax}]
\ProofGTwoTax
\end{proof}

\begin{theorem}[first-order guided AP on the discriminative crossover]\label{thm:cross-ap}
On a discriminative Gaussian coordinate ($a=0<b$), for $w\ge0$ and a uniform
$\lambda$-mesh of step $0<h\le1$, the fitted step \eqref{eq:g2-denoiser} and
the exact guided flow satisfy $\xi_K^{\rm fit}=\Theta_K\,\xi_K^{\rm ex}$ with
$0\le\log\Theta_K\le\tfrac{1+h}{4}\,w(w+2)\,h$, hence
$W_2\bigl(\mathrm{Law}(\xi_K^{\rm fit}),\mathrm{Law}(\xi_K^{\rm ex})\bigr)
\le\bigl(\eexp^{w(w+2)h/2}-1\bigr)(\E|\xi_0|^2)^{1/2}$.  The multiplier bound and
its prefactor are independent of $b$, $K$, and $\smin$ (the $W_2$ scale carries
the initial second moment $(\E|\xi_0|^2)^{1/2}$): on this crossover the fitted
step is first-order uniformly accurate against the exact flow's own pushforward,
though the $O(w^2)$ prefactor is not uniform in $w$.
\end{theorem}
\begin{proof}
\ProofGCrossAP
\end{proof}

\section{The reverse ordering}
\label{app:reverse}

This appendix expands Remark~\ref{rem:reverse}.

\ReverseFull

\section{Metric-split diagnostics}
\label{app:split}

To interpret the FID/KID split at the hard cell, we computed post-hoc
feature-manifold and sharpness diagnostics on the same images.  Fitted CFG
improves Inception precision/recall over CFG in both 5k blocks (e.g.\
$0.799/0.474$ versus $0.704/0.360$); interval has higher recall but lower
precision.  The pixel-saturation median is $0.16$ for CFG and $0$ for fitted
CFG and interval, and the median Laplacian energy drops by an order of
magnitude---consistent with the fitted step removing oversaturated
high-frequency artifacts (we use the Laplacian only as an oversaturation proxy,
not as evidence of sharpness).  The DINOv2 audit below is a backbone-swapped
control on the same split.

\section{DINOv2 feature audit}
\label{app:dino}

To test whether the FID/KID split is Inception-specific, we recompute
feature-space distances with a DINOv2 backbone (\texttt{dinov2\_vits14},
CLS-token features; images bilinearly resized to $224$) against $50$k
CIFAR-10 training images.  We report the Fr\'echet distance FD-DINOv2 and an
unbiased RBF-kernel MMD (median-heuristic bandwidth, $100$ subsets of size
$1000$).  On the two 5k hard-cell blocks, FD-DINOv2 for CFG/fitted/interval is
$469.73/230.00/257.88$ and $469.10/227.86/254.25$, and the DINO-MMD is
$0.0258/0.0126/0.0168$ and $0.0257/0.0125/0.0164$; both favor fitted CFG in
both blocks.  FD-DINOv2 is the clean backbone-only control; the MMD changes
both backbone and kernel relative to Inception KID, so it is corroborating,
not a direct KID analogue.

\section{Tuned limited-interval baselines}
\label{app:interval}

The Section~\ref{sec:checkpoint-certs} interval baseline is the parameter-free
terminal cutoff, which on the hard mesh reduces to unguided conditional DDIM.
To test whether a \emph{tuned} interval that preserves conditionality can match
the fitted step, we selected per cell the conditionality-viable guidance
interval $[\sigma_{\rm lo},\sigma_{\rm hi}]$ from a nine-candidate grid (the one
whose target accuracy stays closest to CFG), then scored it on a fresh 5k block.
Table~\ref{tab:interval} reports CFG, the fitted step, the tuned interval, and
the terminal interval.  In the tested cells, a conditionality-preserving tuned
interval does not beat the fitted step on FID; the terminal interval can win FID
in a cell
($w{=}8,N{=}16$: $12.63$ vs $17.12$) but pays a clear target-accuracy cost.  The
fitted step gives interval-like residual stabilization while keeping the guided
update active.  (Target accuracy is a classifier proxy; the high fitted value at
$N{=}16$ should not be read as fuller conditional fidelity, as it may partly
reflect reduced diversity.)

\begin{table}[t]
\centering
\caption{Tuned vs.\ terminal limited-interval baselines (5k images per cell).
``tuned int.'' is the conditionality-selected interval; ``terminal int.'' is the
parameter-free terminal cutoff of Section~\ref{sec:checkpoint-certs}.  amp p95 is
the 95th-percentile residual amplification, clip p95 the final-denoise clipping
fraction, target acc the classifier-proxy target-label accuracy (\%).}
\label{tab:interval}
{\small
\begin{tabular}{llccccc}
\toprule
cell & scheme & FID $\downarrow$ & KID $\downarrow$ & amp p95 $\downarrow$ & clip p95 $\downarrow$ & target acc $\uparrow$ \\
\midrule
\multirow{4}{*}{$w{=}8,N{=}8$}
 & CFG           & $32.21$ & $0.0105$ & $13.31$ & $0.575$ & $95.5$ \\
 & fitted        & $25.70$ & $0.0180$ & $1.66$  & $0$     & $94.9$ \\
 & tuned int.    & $34.25$ & $0.0205$ & $1.71$  & $0.018$ & $95.3$ \\
 & terminal int. & $26.68$ & $0.0225$ & $1.63$  & $0$     & $89.1$ \\
\midrule
\multirow{4}{*}{$w{=}8,N{=}16$}
 & CFG           & $30.83$ & $0.0117$ & $7.30$  & $0.664$ & $93.8$ \\
 & fitted        & $17.12$ & $0.0070$ & $1.66$  & $0.006$ & $99.3$ \\
 & tuned int.    & $30.18$ & $0.0120$ & $1.75$  & $0.257$ & $98.8$ \\
 & terminal int. & $12.63$ & $0.0072$ & $1.66$  & $0.000$ & $94.8$ \\
\midrule
\multirow{4}{*}{$w{=}6.5,N{=}8$}
 & CFG           & $25.36$ & $0.0057$ & $8.08$  & $0.394$ & $98.4$ \\
 & fitted        & $25.36$ & $0.0175$ & $1.40$  & $0$     & $94.8$ \\
 & tuned int.    & $32.60$ & $0.0208$ & $1.43$  & $0.002$ & $94.7$ \\
 & terminal int. & $26.57$ & $0.0216$ & $1.38$  & $0$     & $89.8$ \\
\bottomrule
\end{tabular}}
\end{table}
\FloatBarrier

\section{Latent-diffusion transfer smoke (Stable Diffusion 1.5)}
\label{app:ldm}

As a cross-domain check that the coefficient is not specific to pixel-space
CIFAR \textsc{edm}, we ran a small Stable Diffusion~1.5 DDIM smoke test at high
guidance, swapping only the terminal guidance coefficient.  We report pixel
saturation (95th percentile of the clipped-pixel fraction) and CLIP
image--text alignment; the interval baseline is terminal guidance shutdown.
Table~\ref{tab:ldm} shows the fitted step sharply reduces saturation while
preserving CLIP alignment far better than shutdown---the same saturation-repair
pattern as on CIFAR.  This is a smoke/deepen test, not a Stable Diffusion
benchmark.

\emph{Setup.}  We use \texttt{runwayml/stable-diffusion-v1-5} at
$512\times512$ with deterministic DDIM ($\eta=0$, $\epsilon$-prediction) and
float16, over a fixed list of everyday-scene captions (repeated to $256$ and
$128$ prompts) with no negative prompt and seeds matched across schemes.  The
fitted step swaps only the terminal guidance coefficient; the interval baseline
disables the guidance correction on the terminal steps below the same
$h_\flat$-derived cutoff.  Saturation p95 is the 95th percentile over images of
the fraction of final RGB pixels at the valid-range extremes; CLIP alignment
uses \texttt{openai/clip-vit-base-patch32}.

\begin{table}[ht]
\centering
\caption{Stable Diffusion~1.5 DDIM smoke test.  Saturation p95 (lower is less
oversaturated); CLIP mean (higher is better text alignment).  ``interval'' is
terminal guidance shutdown.}
\label{tab:ldm}
{\small
\begin{tabular}{llcc}
\toprule
setting & scheme & saturation p95 $\downarrow$ & CLIP mean $\uparrow$ \\
\midrule
\multirow{3}{*}{$g{=}12$, $N{=}12$ (256 img)}
 & CFG      & $0.134$ & $0.320$ \\
 & fitted   & $0.009$ & $0.314$ \\
 & interval & $0.001$ & $0.271$ \\
\midrule
\multirow{3}{*}{$g{=}7.5$, $N{=}20$ (128 img)}
 & CFG      & $0.098$ & $0.321$ \\
 & fitted   & $0.028$ & $0.319$ \\
 & interval & $0.005$ & $0.278$ \\
\bottomrule
\end{tabular}}
\end{table}
\FloatBarrier

\section{All-class residual audit}
\label{app:allclass}

The real-checkpoint residual diagnostic (Table~\ref{tab:real-checkpoint-gate})
uses class~0 and class~1 blocks.  To rule out class cherry-picking, we extended
it to all ten CIFAR-10 classes on the high-guidance cells ($w\in\{6.5,8\}$,
$N\in\{8,16,32\}$; 60 class--cell pairs, all fixed before the run).
Clipping/saturation is non-increased in $60/60$ pairs, and amp p95 improves in
$54/60$; the other six are modestly worse on the amp tail (ratios $1.02$--$1.25$,
largest class~4 at $w{=}6.5,N{=}8$), with clipping still non-increased there.  We
describe this as \emph{global saturation robustness} and \emph{near-global
residual-tail improvement}, and leave the raw counts for the reader to weigh.
\FloatBarrier

\section{Dense guided-flow reference}
\label{app:denseref}

Theorem~\ref{thm:cross-ap} is an accuracy statement against the exact guided
flow's own pushforward on the calibration model.  On learned checkpoints this
does \emph{not} extend to final-image accuracy: measured against a $512$-step
vanilla-CFG reference trajectory, the fitted step is \emph{not} uniformly closer
in final denoised-image $L_2$ (Table~\ref{tab:denseref}, ratios $>1$), even
though its residual and amp tails are far smaller (ratios $\ll1$).  So the
learned-checkpoint evidence supports terminal residual/saturation repair, not a
claim that the fitted step is a uniformly better coarse integrator of the vanilla
CFG field.

\begin{table}[ht]
\centering
\caption{Dense-reference diagnostic: fitted-CFG/CFG ratios against a $512$-step
vanilla-CFG reference ($512$ samples per cell).  Final-image rel-$L_2$ ratios
exceed $1$ (not closer to the dense CFG endpoint); residual rel-$L_2$ and amp p95
ratios are far below $1$.}
\label{tab:denseref}
{\small
\begin{tabular}{lcccc}
\toprule
cell & image rel-$L_2$ med & image rel-$L_2$ p95 & residual rel-$L_2$ p95 & amp p95 \\
\midrule
$w{=}8,N{=}8$  & $1.143$ & $1.109$ & $0.329$ & $0.124$ \\
$w{=}8,N{=}16$ & $1.332$ & $1.170$ & $0.884$ & $0.233$ \\
\bottomrule
\end{tabular}}
\end{table}